\documentclass{article}
\usepackage[utf8]{inputenc}
\usepackage{amsmath,amsthm, amssymb, amsfonts}
\usepackage{todonotes}
\usepackage{kantlipsum}
\usepackage{tikz}
\usetikzlibrary{arrows}
\usepackage{wrapfig}
\usepackage{paralist}
\usepackage{verbatim}
\usepackage{relsize}
\usepackage{hyperref}
\usepackage{graphicx}
\usepackage{epstopdf}
\usepackage{stmaryrd}
\usepackage{wrapfig}
\allowdisplaybreaks[1]
\usepackage{sistyle}
\newcommand\bSI[1]{{\small[\SI{}{#1}]}}

\usepackage{pgfplots}
\usepgfplotslibrary{groupplots}
\usetikzlibrary{arrows.meta}
\usepackage{caption}

\makeatletter
\newlength\unitwdth
\newlength\numwdth
\newlength\tdima
\newcommand\SIdescr[2]{%
    \setlength\tdima{\linewidth}%
    \addtolength\tdima{\@totalleftmargin}%
    \addtolength\tdima{-\dimen\@curtab}%
    \addtolength\tdima{-\unitwdth}%
    \addtolength\tdima{-\numwdth}%
    \parbox[t]{\tdima}{%
        #1
        \leaders\hbox{$\m@th\mkern \@dotsep mu\hbox{\tiny.}\mkern \@dotsep mu$}%
        \hfill
        \ifhmode\strut\fi
        \makebox[0pt][l]{%
            \makebox[\unitwdth][l]{}%
            \makebox[\numwdth][r]{#2}}}}
\makeatother

\newcommand{\N}{\mathbb{N}}

\newcommand{\NN}{\mathcal{N}}

\newcommand{\supp}{\mathrm{supp}}

\newtheorem{theorem}{Theorem}[section]
\newtheorem{remark}[theorem]{Remark}
\newtheorem{definition}[theorem]{Definition}

\newtheorem{lemma}[theorem]{Lemma}

\usepackage{cleveref}
\renewcommand{\epsilon}{\varepsilon}

\newcommand{\M}{\mathcal{M}}
\renewcommand{\L}{\mathcal{L}}

\usepackage[accepted]{icml2020}

\usepackage{amsmath,amsfonts,bm}

\def\eqref#1{equation~\ref{#1}}

\def\1{\bm{1}}

\DeclareMathAlphabet{\mathsfit}{\encodingdefault}{\sfdefault}{m}{sl}
\SetMathAlphabet{\mathsfit}{bold}{\encodingdefault}{\sfdefault}{bx}{n}

\newcommand{\R}{\mathbb{R}}

\usepackage{hyperref}
\usepackage{url}

\icmltitlerunning{Constructive Universal High-Dimensional Distribution Generation through Deep ReLU Networks}

\begin{document}

\twocolumn[
\icmltitle{Constructive Universal High-Dimensional Distribution Generation through Deep ReLU Networks}

\icmlsetsymbol{equal}{*}

\begin{icmlauthorlist}
\icmlauthor{Dmytro Perekrestenko}{ETH}
\icmlauthor{Stephan M\"uller}{ETH}
\icmlauthor{Helmut B\"olcskei}{ETH,ETH2}
\end{icmlauthorlist}

\icmlaffiliation{ETH}{Department of Information Technology and Electrical Engineering, ETH Z\"urich, Z\"urich, Switzerland}
\icmlaffiliation{ETH2}{Department of Mathematics, ETH Z\"urich, Z\"urich, Switzerland}

\icmlcorrespondingauthor{Dmytro Perekrestenko}{pdmytro@mins.ee.ethz.ch}

\icmlkeywords{Machine Learning, Approximation Theory, Generative Networks, ICML}

\vskip 0.3in
]

\printAffiliationsAndNotice{}

\begin{abstract}
We present an explicit deep neural network construction that transforms uniformly distributed one-dimensional noise into an arbitrarily close approximation of 
any two-dimensional Lipschitz-continuous target distribution. The key ingredient of our design is a generalization of the  ``space-filling'' property of sawtooth functions discovered in \cite{Bailey2019}. We elicit the importance of depth---in our neural network construction---in driving the Wasserstein distance between the target distribution and the approximation realized by the network to zero. An extension to output distributions of arbitrary dimension is outlined. Finally, we show that 
the proposed construction does not incur a cost---in terms of error measured in Wasserstein-distance---relative to generating $d$-dimensional target distributions from $d$ independent random variables.
\end{abstract}

\section{Introduction}

Deep neural networks have been used very successfully as generative 
models for complex natural data such as images \cite{radford2015unsupervised,karras2018stylebased} and natural language \cite{bowman2015generating, xu2018dpgan}.
Specifically, the idea is to learn the parameters of deep networks \cite{Welling2014:Gan, NIPS2014_5423}
so that they realize complex high-dimensional probability distributions by transforming samples taken from simple low-dimensional distributions such as uniform or Gaussian.

Generative networks with higher output than input dimension occur, for instance, in language modelling where deep networks are used to 
predict the next word in a text sequence. Here, 
the input layer size is determined by the dimension of the word embedding (typically $\sim 100$) and the output layer, representing a vector of probabilities for each of the words in the vocabulary, is of the size of the vocabulary (typically $\sim 100k$). Another example where the dimensionality of the input distribution is mandated to be lower than that of the output distribution is 
given by the variational inference methods according to \cite{Welling2014:Gan,Tolstikhin2018Wassauto}. 

Notwithstanding the practical success of deep generative networks, a profound theoretical understanding of their representational capabilities is still lacking. First results along those lines appear in  \cite{lee2017ability}, which establishes that generative networks can approximate distributions arising from the composition of Barron functions \cite{barron1993}.

Bailey and Telgarsky \cite{Bailey2019} show how deep ReLU networks can be used to increase the dimensionality of uniform distributions and how a univariate uniform distribution can be turned into a univariate Gaussian distribution and vice versa. Finally, \cite{lu2020universal} shows that neural networks constitute universal approximators for continuous probability distributions when source and target distribution are of the same dimension. 

Classical approaches for generating multi-dimensional random variables of a given distribution such as the Box-Muller method \cite{box1958} or conditional distribution, rejection, and composition methods \cite{distgenBook} are all based on transforming initial distributions of the same dimensionality as the target distribution. We are not aware of methods that map one-dimensional inputs to prescribed $d$-dimensional outputs. The purpose of the present paper is to show that deep generative networks are capable of doing exactly that and moreover are also universal generators, in contrast to, e.g., the Box-Muller method \cite{box1958}, which maps uniform distributions to Gaussian distributions, albeit with zero error.
We also quantify how the connectivity of the resulting networks scales with the approximation error measured in Wasserstein distance.

The problem is approached in two steps. Specifically, given a two-dimensional Lipschitz-continuous target distribution, we first find the (two-dimensional) histogram distribution that best approximates it---for a given histogram resolution---in Wasserstein distance. The resulting histogram distribution is then realized by a ReLU network driven by a uniform univariate input distribution. To this end, we develop a new space-filling property of ReLU networks, generalizing that discovered in \cite{Bailey2019}. The main conceptual insight of this paper is that generating arbitrary $d$-dimensional target distributions, with $d\ge 2$, from a one-dimensional uniform distribution through a deep neural network does not come at a cost---in terms of approximation error measured in Wasserstein distance---relative to generating the target distribution from $d$ independent random variables.
We emphasize that the generating network has to be deep, in fact the depth has to go to infinity to obtain the same error in Wasserstein-distance as a construction from $d$ independent random variables would yield.

We finally note that our results pertain only to representational capabilities of generative (ReLU-)networks and we do not consider the problem of learning the network weights and biases.

\subsection{Notation and Definitions} We denote the set of integers in the range $[1,n]$ by $[[1,n]]$. $U(\Delta)$ stands for the uniform distribution on the interval $\Delta$, when $\Delta=[0,1]$, we simply write $U$.
Given a probability distribution with pdf $p$, we denote the push-forward of $p$ under the function $f$ as $f\#p$. For a given compact set ${\cal C}$, we let $p_{\mathbf{X}}(\mathbf{x} \in {\cal C}) = \int_{\cal C} p_{\mathbf{X}}(\mathbf{x}) d\mathbf{x}$.
We define ReLU neural networks as follows.
\begin{definition}\label{def:NN}
Let $L, N_0, N_1, \ldots, N_{L}\in \N$, $L\geq 2$. A map $\Phi: \R^{N_0} \to \R^{N_L}$ given by
\begin{equation*}\label{eq:NNdef}
\Phi(x) = \begin{cases}
\begin{array}{lc} \hspace{-0.1cm}W_2(\rho \, (W_1(x))), & L=2\\
\hspace{-0.1cm} W_L(\rho \, ( W_{L-1} (\rho\, ( \dots \rho \, ( W_{1}(x)))))), & L\ge3 
\end{array}, \end{cases}
\end{equation*}
with affine linear maps $W_{\ell}: \R^{N_{\ell-1}} \to \R^{N_\ell}$, $\ell\in\{1,2,\dots,L\}$, and the ReLU activation function $\rho(x) = \max(x,0), \ x \in \mathbb{R}$, acting component-wise (i.e., $\rho(x_1,\dots,x_N):=(\rho(x_1),\dots,\rho(x_N))$) is called a \textnormal{ReLU neural network}. The map $W_{\ell}$ corresponding to layer $\ell$ is given by $W_{\ell}(x)=A_\ell x + b_\ell$, with $A_{\ell}\in \mathbb{R}^{N_{\ell}\times N_{\ell-1}}$ and $b_\ell\in \mathbb{R}^{N_\ell}$. %
We define the \emph{network connectivity} $\M(\Phi)$ as the total number of non-zero entries in the matrices $A_\ell$, $\ell\in\{1,2,\dots,L\}$, and the vectors $b_\ell$, $\ell\in\{1,2,\dots,L\}$. The \emph{depth of the network} or, equivalently, the number of layers is $\L(\Phi):=L$ and its width is given by $\mathcal{W}(\Phi):=\max_{\ell=0,\dots,L} N_\ell$. We denote by $\NN_{d,d'}$ the set of ReLU networks with input dimension $N_0=d$ and output dimension $N_L=d'$.\end{definition}

We measure the distance between distributions in terms of Wasserstein distance defined as follows.
\begin{definition}
Let $\mu$ and $\nu$ be distributions on $\mathbb{R}^d$ and denote the set of distributions on $\mathbb{R}^d \times \mathbb{R}^d$ whose first and second marginals coincide with $\mu$ and $\nu$, respectively, by $\prod(\mu,\nu)$. Then, the Wasserstein distance between $\mu$ and $\nu$ is defined as \[ W(\mu,\nu):= \inf_{\pi \in \prod(\mu,\nu)} \int |x-y| d \pi(x,y),  \]
where the elements of the set $\prod(\mu,\nu)$ are called couplings of $\mu$ and $\nu$.
\end{definition}

\begin{definition}
For distributions $\mu$ and $\nu$ on $\mathbb{R}^d$ with corresponding pdfs $p_\mu, p_\nu$ supported on $\Omega \subset \mathbb{R}^d$, the total variation (TV) distance is defined as \[ TV(\mu,\nu):= \frac{1}{2} ||p_\mu - p_\nu||_{L_1(\Omega)}. \]
\end{definition}

The following relation between Wasserstein distance and TV-distance was found in \cite{Gibbs2002}.
\begin{theorem}{\cite{Gibbs2002}}
\label{gibbs}
For distributions $\mu$ and $\nu$ on $\mathbb{R}^d$ with pdfs $p_\mu, p_\nu$ supported on $\Omega \subset \mathbb{R}^d$, the Wasserstein distance and the TV-distance satisfy 
\[ W(\mu,\nu) \leq \text{diam}(\Omega) \cdot TV(\mu,\nu),\]
where $\text{diam}(\Omega) = \sup \{ |x-y| : x,y \in \Omega \}$.
\end{theorem}

Next, we define $d$-dimensional histogram distributions.

\begin{definition}
\label{gen_d_dim_dist}
A random vector $\mathbf{X} = (X_1, X_2, \dots, X_d)$ is said to have a general histogram distribution of resolution $n$ on the $d$-dimensional unit cube, denoted as $\mathbf{X} \sim \mathcal{G}[0,1]^d_n$, if for some $0=t^j_0<t^j_1<\dots<t^j_n=1$, $j \in [[1,d]]$, its pdf is given by
\[
\begin{aligned}
p(\mathbf{x}) &= \sum_{\mathbf{k}} w_{\mathbf{k}} \chi_{c_{\mathbf{k}}}(\mathbf{x}), \quad \sum_{\mathbf{k}} w_{\mathbf{k}} \prod_{j=1}^d(t^j_{i_j+1}-t^j_{i_j}) = 1, \\ &w_{\mathbf{k}}>0, \ \ \text{for all} \ \ \mathbf{k} \in [[0,n-1]]^d,
\end{aligned}
\]
where $\mathbf{k} = (i_1, i_2, \dots, i_d) \in [[0,n-1]]^d$ is an index vector and $\chi_{c_{\mathbf{k}}}(\mathbf{x})$ is the characteristic function of the $d$-dimensional cube $c_{\mathbf{k}} = [t^1_{i_1}, t^1_{i_1+1}] \times [t^2_{i_2}, t^2_{i_2+1}] \times \dots \times[t^d_{i_d}, t^d_{i_d+1}]$.
\end{definition}

We will mostly be concerned with histogram distributions of uniform tile size, defined as follows.
\begin{definition}
\label{d_dim_dist}
A random vector $\mathbf{X} = (X_1, X_2, \dots, X_d)$ is said to have a histogram distribution of resolution $n$ on the $d$-dimensional unit cube, denoted as $\mathbf{X} \sim \mathcal{E}[0,1]_n^d$, if its pdf is given by
\[
\begin{aligned}
p(\mathbf{x}) &= \sum_{\mathbf{k}} w_{\mathbf{k}} \chi_{c_{\mathbf{k}}}(\mathbf{x}), \quad \sum_{\mathbf{k}} w_{\mathbf{k}}  = n^d, \\ &w_{\mathbf{k}}>0, \ \ \text{for all} \ \ \mathbf{k} \in [[0,n-1]]^d,
\end{aligned}
\]
where $\mathbf{k} = (i_1, i_2, \dots, i_d) \in [[0,n-1]]^d$ is an index vector and $\chi_{c_{\mathbf{k}}}(\mathbf{x})$ is the characteristic function of the $d$-dimensional cube $c_{\mathbf{k}} = [i_1/n, (i_1+1)/n] \times [i_2/n, (i_2+1)/n] \times \dots \times [i_d/n, (i_d+1)/n]$.
\end{definition}

\begin{remark}
 For ease of exposition, in Definitions \ref{gen_d_dim_dist} and \ref{d_dim_dist}, we let $c_{\mathbf{k}}$ be a product of closed intervals, thus allowing the breakpoints to belong to different cubes. While this comes without loss of generality, for concreteness, it is understood that
 the value of the pdf at the breakpoints is the average across the cubes the corresponding breakpoint belongs to.
\end{remark}

\section{Universal approximation}

As mentioned in the introduction, the intermediate step in our construction consists of a ReLU network that turns a univariate one-dimensional input distribution into a two-dimensional histogram distribution. This histogram distribution is then chosen such that it approximates the two-dimensional Lipschitz-continuous target distribution.
To understand why we chose this two-step approach, note that ReLU networks generate piecewise linear functions and the pushforward $f\# U$ of any
piecewise linear $f:\R\rightarrow\R$ yields a histogram distribution.
We start by quantifying the TV distance between an arbitrary distribution and a histogram distribution of resolution $n$.
\begin{theorem}
\label{dict_theorem}
Let $p$ be a $d$-dimensional L-Lipschitz-continuous pdf of finite differential entropy on its support $[0,1]^d$. Then, for every $n>0$, there exists a $\tilde{p} \in \mathcal{E}[0,1]_n^d$ such that
\[ TV(p, \tilde{p}) = \frac{1}{2}\| p - \tilde{p}\|_{L_1([0,1]^d)} \leq \frac{L\sqrt{d}}{2n}.\]
\end{theorem}
\begin{proof}
The proof is based on the Mean 
Value Theorem, which states that, for any continuous $d$-dimensional function $p$ supported on $\Omega \in \mathbb{R}^d$, there exists a $\mathbf{z} \in \Omega$, such that
\begin{equation}
\label{mean_value_theorem}
\int_{\Omega} p(\mathbf{x}) d\mathbf{x} = p(\mathbf{z})\int_{\Omega} d\mathbf{x}.    
\end{equation}
Next, we divide the unit cube $[0,1]^d$ into the $n^d$ cubes $c_{\mathbf{k}}$ per Definition \ref{d_dim_dist}.
Take an arbitrary $\mathbf{k} \in [[0,n-1]]^d$ and fix $\mathbf{z}_\mathbf{k}$ according to Equation \ref{mean_value_theorem}
with $\Omega=c_{\mathbf{k}}$. Then, using the Lipschitz property of $p$, we obtain
\[
\begin{aligned}
&\| p(\mathbf{x}) - p(\mathbf{z}_\mathbf{k}) \|_{L_1(c_\mathbf{k})} = \int_{c_\mathbf{k}} |p(\mathbf{x}) - p(\mathbf{z}_\mathbf{k})| d\mathbf{x} \\ &\leq \int_{c_\mathbf{k}} L |\mathbf{x} - \mathbf{z}_\mathbf{k}| d\mathbf{x} \leq \int_{c_\mathbf{k}} L \frac{\sqrt{d}}{n} d\mathbf{x} = L \frac{\sqrt{d}}{n} \cdot \frac{1}{n^d}. 
\end{aligned}
\]
We set  
\[\tilde{p}(\mathbf{x}) = \sum_{\mathbf{k}} p(\mathbf{z}_\mathbf{k}) 
\chi_{c_{\mathbf{k}}}(\mathbf{x})\]
and note that $\tilde{p} \in \mathcal{E}[0,1]_n^d$ as $\sum_{\mathbf{k}} p(\mathbf{z}_\mathbf{k}) = n^d$ owing to Equation \ref{mean_value_theorem}; moreover, $p(\mathbf{z}_\mathbf{k})>0$, for all $\mathbf{k}$, as $p$ is of finite differential entropy on
$[0,1]^d$.
Finally, summing up across all cubes $c_{\mathbf{k}}$, we obtain
\[
\begin{aligned}
\| p - \tilde{p}\|_{L_1{([0,1]^d)}} &= \int_{[0,1]^d} |p(\mathbf{x}) - \tilde{p}(\mathbf{x})| d\mathbf{x} \\ &\leq \sum_{\mathbf{k}} \int_{c_\mathbf{k}} L |\mathbf{x} - \mathbf{z}_\mathbf{k}| d\mathbf{x}  \leq L \frac{\sqrt{d}}{n}. \hspace{0.3cm}\qedhere
\end{aligned}
\]
\end{proof}
Henceforth, we shall always assume that probability density functions $p$ are of finite differential entropy on their support, without explicitly 
declaring it.

We are now ready to state the main result of the paper, the proof of which is largely based on Theorem \ref{2d_theorem} below.

\begin{theorem}
\label{wasserstein-distance}
Let $p_{X,Y}$ be an $L$-Lipschitz-continuous pdf supported on $[0,1]^2$. Then, for every $n>0$, there exists a $\Phi \in \NN_{1, 2}$ with connectivity $\mathcal{M}(\Phi)\leq 88 (n^2 + ns)$ and of depth $\mathcal{L}(\Phi) = s+5$, such that
\[ W(\Phi \#U, p_{X,Y} ) \leq \frac{L\sqrt{2}}{2n} + \frac{2\sqrt{2}}{n2^s}.\]
\end{theorem}
\begin{proof}

Combining Theorem \ref{dict_theorem} with Theorem \ref{gibbs}, we obtain that for every $n>0$, there exists a $\tilde{p} \in \mathcal{E}[0,1]_{n}^2$ such that
\[ W(p, \tilde{p}) \leq \frac{L}{n}.\]
On the other hand, it follows from Theorem \ref{2d_theorem} that, for every $ \tilde{p} \in \mathcal{E}[0,1]_{n}^2$, there exists a neural network $\Phi \in \NN_{1, 2}$ with connectivity $\mathcal{M}(\Phi)\leq 88 (n^2 + ns)$ and of depth $\mathcal{L}(\Phi) = s+5$ such that
\[ W(\Phi \#U, \tilde{p}  ) \leq \frac{2\sqrt{2}}{n2^s}. \]
We finalize the proof by application of the triangle inequality for Wasserstein distance \cite{triWass} to get
\[
\begin{aligned}
W(\Phi \#U, p ) &\leq W(\Phi \#U, \tilde{p} ) \\&+  W(p, \tilde{p}) = \frac{L}{n} + \frac{2\sqrt{2}}{n2^s}.\hspace{1.2cm}\qedhere
\end{aligned}
\]
\end{proof}
The error bound in Theorem \ref{wasserstein-distance} illustrates the main conceptual insight of this paper, namely that generating arbitrary two-dimensional Lipschitz-continuous distributions from a one-dimensional uniform distribution through a deep neural network does not come at a cost---in terms of Wasserstein-distance error---relative to generating this two-dimensional target distribution from two independent random variables. 
Specifically, if we let the depth $s$ of the generating network go to infinity, the second term in the error bound will go to zero exponentially fast in $s$ leaving us only with the first term, which reflects the error stemming from the histogram approximation of the distribution. Moreover, this first term is inversely proportional to the histogram resolution $n$ and linear in the Lipschitz constant and can thus be made arbitrarily small by letting the histogram resolution $n$ approach infinity. The width of the corresponding generating network will grow according to $n^2$. When the target distribution is uniform, we recover the result in \cite{Bailey2019}. The intermediate step via histogram distributions was not needed in \cite{Bailey2019} as Bailey and Telgarsky only considered mapping uniform input distributions to uniform output distributions. Finally, we note that our result carries over to general $d$-dimensional output distributions; we briefly comment on this extension in Section \ref{higher-dim}.

\section{ReLU networks and histograms}

This section systematically establishes the connection between ReLU networks and histogram distributions. Specifically, we show that the pushforward of a uniform distribution under a piecewise linear function results in a histogram distribution. We will also identify, for a given histogram distribution, the corresponding piecewise linear function generating it under pushforward of a uniform distribution. Combined with the insight that ReLU networks always realize piecewise linear functions, we will have established the desired connection.

We start with a simple auxiliary result.

\begin{lemma}
\label{one-piece}
Let $a,b \in \R, a<b, \Delta = [a,b]$, and let $h(x)=mx+s$, for $x \in \R$, with $m \in \R \setminus \{0\}, s \in \R$. Then, $Q = h\#U(\Delta)$ is uniformly distributed on $[ma+s,mb+s]$, for $m>0$, and on $[mb+s,ma+s]$, for $m<0$.
\end{lemma}
\begin{proof}
The pdf of the pushforward of a general random variable with pdf $p(x)$ under the general function $h(x)$ is
\begin{equation*}
    \label{push-forward_formula}
    q(y)= p(h^{-1}(y)) \left|\frac{d}{dy}h^{-1}(y)\right|.
\end{equation*}
Particularized to $h^{-1}(y) = \frac{y-s}{m}$ and $p(x) = \frac{1}{b-a} \chi_\Delta(x)$, this yields
\[ q(y) = \begin{cases} \frac{1}{m(b-a)}, & \text{if } y \in [ma+s, mb+s] \\
0, & \text{otherwise}
\end{cases}\]
for $m>0$, and 
\[ q(y) = \begin{cases} \frac{1}{|m|(b-a)}, & \text{if } y \in [mb+s, ma+s] \\
0, & \text{otherwise}
\end{cases}\]
for $m<0$.
\end{proof}

We next show that the pushforward of a uniform distribution under a piecewise linear function always results in a histogram distribution.

\begin{theorem}
\label{pwl2histogram}
For any piecewise linear continuous function $f:\R\rightarrow\R$, such that $f(x) \in [0,1], \forall x \in [0,1]$, and $f(0)=0, f(1)=1$, there exists an $n$, such that $f\# U \in \mathcal{G}[0,1]^1_n$.
\end{theorem}
\begin{proof}
As $f$ is piecewise linear, we can split its support interval into $t\in \N$ intervals $I_i, i \in [[0,t-1]]$, on which it is linear. We hence have $\bigcup_{j=0}^{t-1} I_j = \text{supp}(f)$. The pdf of $q = f \#U$ can now be computed by conditioning on $U$ being in the interval $I_j$ and summing up the contributions of the individual intervals. Using the law of total probability and the chain rule, we find that
\[ q(y) = \sum_{j=0}^{t-1} q(y|u \in I_j) \mathbb{P}(u \in I_j).\]
As $U$ is uniform, it is also uniform conditional on being in a given interval $I_j$.
By Lemma \ref{one-piece} it therefore follows that $q(y|x \in I_j)$ is uniform, $\forall j \in [[0,t-1]]$, and can be written as $q(y|x\in I_j) = \frac{\chi_{R_j}}{|R_j|}$, for some interval $R_j \subseteq [0,1]$. Setting $w_j = \mathbb{P}(x \in I_j)$,
the density $q(y)$ thus has the form
\[ q(y) = \sum_{j=0}^{t-1} w_j \frac{\chi_{R_j}}{|R_j|}.\]
By continuity of $f$ and the boundary conditions 
$f(0)=0,f(1)=1$, we know that $\bigcup_j R_j = [0,1]$. Since $q(y)$ is a step function, there exists a histogram resolution $n$ such that 
$q(y) \in \mathcal{G}[0,1]^1_n$.
\end{proof}

We will also need the converse to the result just established, in particular a constructive version thereof explicitly identifying the 
piecewise linear function that leads to a given histogram distribution under pushforward of a uniform distribution on the interval $[0,1]$.

\begin{theorem}
\label{1d_theorem}
Let $p_X(x)$ be a pdf in $\mathcal{G}[0,1]^1_n$ with weights $w_k$, $k \in [[0,n-1]]$, and breakpoints $0=t_0<t_1<\dots<t_n=1$, and let $a_0 = \frac{1}{w_{0}}, \ a_i = \frac{1}{w_{i}} - \frac{1}{w_{i-1}}$, $b_0=0$, $b_i = \sum_{j=0}^{i-1} (t_{j+1} - t_j) w_j$, $i \in [[1,n]]$. Then, 
\[f(x) = \sum_{i=0}^{n-1} a_i \max(0,x-b_i)\]
is the piecewise linear map
satisfying $f\#U = p_X(x)$.
\end{theorem}
\begin{proof}
Let $I_i:=[b_i, b_{i+1}]$, $i \in [[0,n-1]]$. 
Then, $\bigcup_{i\in[[0,n-1]]} I_i = [0,1]$ and for all $i \in [[0,n-1]]$, the function $f(x)$ is linear on $I_i$ with slope 
equal to $\sum_{j=0}^{i} a_j=1/w_i$. Next, note that the interval $I_{i}$ is mapped under $f(x)$ to the interval
$I_i^{(1/w_i)} =[f(b_i),f(b_i) + \frac{(b_{i+1}-b_i)}{w_i}]=[t_i,t_{i+1}]$. 
The proof is concluded upon observing that by Lemma \ref{one-piece}, the pdf value of $f\#U$ corresponding to the linear piece $I_i$ equals $\frac{1}{\frac{1}{w_i}}=w_i$.
\end{proof}

We finally note that ReLU networks always realize piecewise linear functions and hence when pushing forward uniform distributions produce histogram distributions. This extends to arbitrary dimensions, i.e., for any ReLU network $\Phi \in \NN_{d,d'}$, the pushforward $\Phi\#U[0,1]^d$ results in a histogram distribution.

\section{Generating two-dimensional distributions with ReLU networks}

We next develop a new space-filling property of ReLU networks, generalizing the one discovered in \cite{Bailey2019}, and then show how
this idea can be used to produce arbitrarily accurate approximations of two-dimensional histogram distributions through deep neural networks
driven by univariate uniform input distributions. 

Our construction is based on
higher-order sawtooth functions obtained as follows. 
Consider the sawtooth function $g: [0,1] \rightarrow [0,1]$,
\begin{equation*}
g(x) = \begin{cases} 
2x, &\mbox{if } x < \frac{1}{2}, \vspace{0.1cm}\\
2(1-x), &\mbox{if } x \geq \frac{1}{2}, \\ 
\end{cases}
\end{equation*}
let $g_{1}(x)=g(x)$, and define the ``sawtooth'' function of order $s$ as the $s$-fold composition of $g$ with itself according to
\begin{equation}
\label{g_s_def}
g_s := \underbrace{g \circ g \circ \dots \circ g}_{s},
\hspace{0.3cm} s \geq 2.
\end{equation}
Next, we note that $g$ can be realized by a $2$-layer ReLU network $\Phi_g \in \NN_{1,1}$ of connectivity $\mathcal{M}(\Phi_g)=8$ according to 
$\Phi_g(x) = W_2(\rho(W_1(x))=g(x)$
with
\[W_1(x) =  \hspace{-0.1cm} 
 \begin{pmatrix}
  1\\
  1\\
  1
 \end{pmatrix}  
 \hspace{-0.1cm}x \
 -  
 \begin{pmatrix}
  0 \\
  1/2 \\
  1
 \end{pmatrix} \hspace{-0.1cm},
 \hspace{0.05cm}
 W_2(x) = \begin{pmatrix}
  2 & -4 & 2
 \end{pmatrix}  \hspace{-0.1cm} \begin{pmatrix}
 x_1\\
 x_2\\
 x_3
 \end{pmatrix}\hspace{-0.1cm}.
\]

The $s$-order sawtooth function $g_s$ can hence be realized by a ReLU network $\Phi^s_g \in \NN_{1,1}$ with connectivity $\mathcal{M}(\Phi) = 11s-3$, and of depth $\mathcal{L}(\Phi)=s+1$ according to $\Phi^s_g(x) = W_2(\rho( \underbrace{W_g(\rho(\dots W_g}_{s-1}(\rho(W_1(x)))))))=g_s(x)$ with
\[W_g(x) = \begin{pmatrix}
  2 & -4 & 2\\
  2 & -4 & 2\\
  2 & -4 & 2
 \end{pmatrix} \begin{pmatrix}
 x_1\\
 x_2\\
 x_3
 \end{pmatrix}
 -
 \begin{pmatrix}
  0 \\
  1/2 \\
  1
 \end{pmatrix}.
\]

Next, we need an auxiliary result on the pushforward---under shifted and scaled versions of $g$---of uniformly distributed random variables.
\begin{lemma}
\label{p_saw0}
Fix $p_X \in \mathcal{E}[0,1]^1_n$ with weights $w_k$ and let $f$ be the piecewise linear function according to Theorem \ref{1d_theorem}, such that $f\#U=p_X$. Fix $H \in \N$, $0<a<b$, $\Delta=[a,b]$, and let $c^i_h:= [i/n+h/H,i/n+(h+1)/H]$, $i \in [[0,n-1]]$, $h \in [[0,H-1]]$. Then, $\Big(f(g((\cdot-a)/(b-a)))\#U(\Delta)\Big) (x\in c^i_h) = p_X(x \in c^i_h) = w_i/H$, for all $i\in [[0,n-1]],$ $h \in [[0,H-1]]$.
\end{lemma}
\begin{proof}
Follows from the symmetry of $g(x)$ and the proof of Theorem $\ref{1d_theorem}$.
\end{proof}

The following result constitutes an important technical ingredient of our space-filling idea.

\begin{lemma}
\label{p_saw}
Let $f(x)$ be a continuous function on $[0,1]$, with $f(0)=0$. Then, for all $s \in \N$,
\[f(g_s(x)) = \sum^{2^{s-1}-1}_{k=0}  f\big(g(2^{s-1}x-k)\big),\]
and for all $k \in [[0,2^{s-1}-1]]$,
\[\supp\big(f\big(g(2^{s-1}x-k)\big)\big) = \Bigg(\frac{k}{2^{s-1}}, \frac{k+1}{2^{s-1}}\Bigg).\]
\end{lemma}
\begin{proof}
We first note that $s$-order sawtooth functions satisfy \cite{telgarsky2016benefits}
\[g_s(x) = \sum^{2^{s}-1}_{k=0} g(2^{s-1}x-k),\]
with $g(2^{s-1}x-k)$ supported in $\left(\frac{k}{2^{s-1}}, \frac{k+1}{2^{s-1}}\right)$. Since $f(0)=0$, the support of $f(g(2^{s-1}x-k))$ coincides with the support of $g(2^{s-1}x-k)$. Hence,
\[
\begin{aligned}
f(g_s(x)) &= f\bigg(\sum^{2^{s}-1}_{k=0} g(2^{s-1}x-k)\bigg) \\&= \sum^{2^{s-1}-1}_{k=0}  f\big(g(2^{s-1}x-k)\big).\hspace{2.3cm}\qedhere
\end{aligned}
\]
\end{proof}

We next present a result showing that two-dimensional histogram distributions that are constant with respect to one of its dimensions, 
can be realized efficiently by deep ReLU networks.
\begin{figure}
  \includegraphics[width=\linewidth]{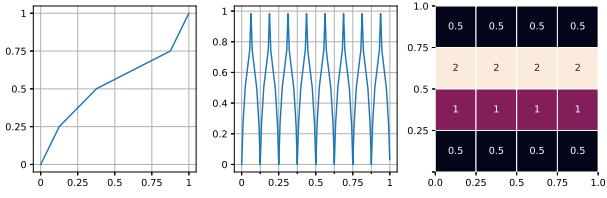}
  \caption{Generating a histogram distribution via the transport map $(x,f(g_s(x)))$. Left---the function $f(x)$, center---$f(g_4(x))$, 
  right---a heatmap of the resulting histogram distribution.}
  \label{fig:map}
\end{figure}

\begin{theorem}
\label{line-wise_theorem}
For any $p_{X,Y}(x,y) \in \mathcal{E}[0,1]^2_n$ with weights $w_{k_1,k_2}=w_{k_2}$, $k_1,k_2 \in [[0,n-1]]$, there exists a $\Phi \in \NN_{1,2}$ with connectivity $\mathcal{M}(\Phi) \leq 6n+24s+2$ and of depth $\mathcal{L}(\Phi) = s+3$, such that
\begin{equation*}
    \label{line_error_bound}
    W(\Phi \#U, p_{X,Y} ) \leq \frac{2\sqrt{2}}{2^s}.
\end{equation*}

\end{theorem}

The transport map realized by the network in Theorem \ref{line-wise_theorem} is based on the generalized space-filling construction $f(g_s(x))$, which
has ``teeth'' in the form of $f(x)$. For an illustration see Figure \ref{fig:map}.

Now consider a general histogram distribution $p_{X,Y}(x,y)$ in
$\mathcal{E}[0,1]_n^2$. We make use of the fact that the marginals and the conditional distributions of a two-dimensional histogram distribution are (one-dimensional) histogram distributions and realize $p_{X,Y}(x,y)$ according to $p_{X,Y}(x,y)=p_{X}(x) \sum_{i=0}^{n-1} p_{Y|X}(y| x \in [i/n,(i+1)/n])$. The formal statement is as follows.

\begin{figure}
  \includegraphics[width=\linewidth]{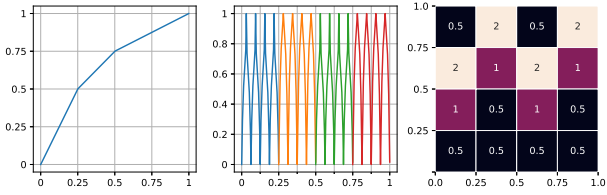}
  \caption{Generating a general $2$-D histogram distribution. Left---the function $f_1=f_3$, center---$\sum_{i=0}^{3} f_i\Big(g_3\Big(4x-i)\Big)\Big)$, right---a heatmap of the resulting histogram distribution.
  The function $f_0=f_2$ is depicted on the left in Figure \ref{fig:map}.}
  \label{fig:map_comb}
\end{figure}

\begin{theorem}
\label{2d_theorem}
For every distribution $p_{X,Y}(x,y)$ in $\mathcal{E}[0,1]_n^2$, there exists a $\Psi \in \NN_{1,2}$ with connectivity
$\mathcal{M}(\Psi)< 88(n^2+ns)$ and of depth $\mathcal{L}(\Psi) = s+5$,
such that
\begin{equation*}
\label{error_bound_2d}
   W(\Phi \#U, p_{X,Y} ) \leq \frac{2\sqrt{2}}{n2^s}.  
\end{equation*}
\end{theorem}

The transport map realized by the network in Theorem \ref{2d_theorem} effectively implements a weighted sum of localized transport maps according to Theorem \ref{line-wise_theorem} and corresponding to the marginals $p_{Y}(y| x \in [i/n,(i+1)/n]), i=[[0,n-1]]$. For an illustration see Figure \ref{fig:map_comb}.

We remark that choosing $s \sim n$, makes the error in Theorem \ref{2d_theorem} decay exponentially in $n$ while 
the connectivity of the network is in $\mathcal{O}(n^2)$; this behavior is asymptotically optimal as the number of parameters in $\mathcal{E}[0,1]^2_n$ is of the same order. Moreover,
we note that Theorem \ref{2d_theorem} generalizes
\cite{Bailey2019}[Theorem 2.1] from uniform target distributions to arbitrary ones through the histogram approximation method and the novel space-filling transport map construction developed in the proof of Theorem~\ref{line-wise_theorem}. This construction can be interpreted as a transport operator in the sense of optimal transport theory \cite{MAL-073,villani2008optimal}, with the source distribution being one-dimensional and the target-distribution two-dimensional.

\section{Higher dimensions} \label{higher-dim}

The extension of our main result to target distributions of dimension higher than $2$ follows the same general storyline as our $2$-D results above, i.e., we approximate the target distribution by a histogram distribution, realize this histogram distribution through a transport map, and then show how this transport map can be implemented by a deep ReLU network. The transport map our extension is based on does not follow as a generalization of that for the 2-D case, but is based on an alternative idea.

\begin{theorem}
\label{dd_theorem_main_nn}
Let $d,n \in \mathbb{N}$. For every $p_{\mathbf{X}} \in \mathcal{E}[0,1]_n^d$,
there exists a $\Psi \in \NN_{1,d}$ with connectivity
$\mathcal{M}(\Psi)\leq 22\cdot2^d(n^d+n^{d-1}s)$ and of depth $\mathcal{L}(\Psi) = (d-1)(s+3)+2$,
such that
\[ W(\Psi \#U[0,1], p_{\mathbf{X}} ) \leq \frac{\sqrt{d}}{n2^s}.\]
\end{theorem}

The transport map underlying this result is based on the following functions. 
Let $s\in \mathbb{N}$, $\Delta=[a,b] \subset [0,1]$, set $\widetilde{b} = a + \frac{2^s(b-a)}{1+2^s}$, and define
\begin{equation*}
        \label{g_int}
        g_{s}^{\Delta}(x):= \frac{1}{n}g_s\left(\frac{x-a}{b-a}\right),
    \end{equation*}
    
 \begin{equation*}
    \label{gr_int}
    h_{s}^{\Delta}(x):= g_{s}^{\Delta}\left(\frac{x-a}{\widetilde{b}-a}\right) + \frac{1}{n(b - \widetilde{b})}(\rho(x-\widetilde{b}) - \rho(x-b)).
    \end{equation*}
\begin{center}
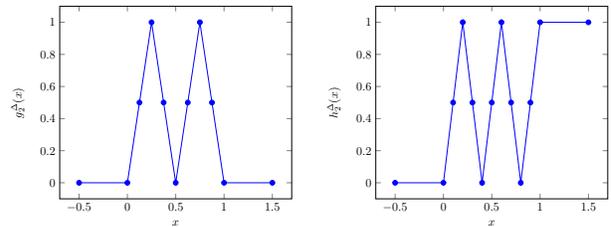

\begin{tabular}{ c c }
 
 \begin{tikzpicture}[scale=0.45]
	\begin{axis}[
		xlabel=$x$,ylabel=$g_{2}^{\Delta}(x)$]
	\addplot[color=blue,mark=*] coordinates {
	    (-0.5,0)
	    (0,0)
		(0.125,0.5)
		(0.25,1)
		(0.375,0.5)
		(0.5,0)
		(0.625,0.5)
		(0.75,1)
		(0.875,0.5)
		(1,0)
		(1.5,0)
		
	};
	\end{axis}%
\end{tikzpicture} & \begin{tikzpicture}[scale=0.45]
	\begin{axis}[
		xlabel=$x$,ylabel=$h_{2}^{\Delta}(x)$]

	\addplot[color=blue,mark=*] coordinates {
	    (-0.5,0)
		(0,0)
		(0.1,0.5)
		(0.2,1)
		(0.3,0.5)
		(0.4,0)
		(0.5,0.5)
		(0.6,1)
		(0.7,0.5)
		(0.8,0)
		(0.9,0.5)
		(1,1)
		(1.5,1)
	};
	\end{axis}%
\end{tikzpicture}\\
\end{tabular}
\captionof{figure}{Plots of $g_{s}^{\Delta}(x)$ (left) and $h_{s}^{\Delta}(x)$ (right) with $n=1,a=0,b=1,s=2$.}
\label{g_int_visualized}
\end{center}

Rather than providing the full details, which are notationally very cumbersome, for illustration purposes, we specify the
transport map for the special case $d=2$ and $n = 2^k$, for some $k\in \N$. 

Let $p_{X,Y}(x,y) \in \mathcal{E}[0,1]^2_n$ have weights $w_{\mathbf{k}}$ and denote the piecewise linear function corresponding to the marginal histogram distribution $p_X$ according to Theorem \ref{1d_theorem} by
$f_{\text{marg}}$. Note that the marginal histogram has weights $w_k=\frac{1}{n}\sum_{i=0}^{n-1} w_{k,i}$. Let $\Delta_{S_{\mathbf{k}}}:=[\frac{1}{n^2} \sum_{\mathbf{y}:S_{\mathbf{y}} < S_{\mathbf{k}}} \frac{w_{\mathbf{y}}}{w_{y_1}}, \frac{1}{n^2}\sum_{\mathbf{y}:S_{\mathbf{y}} \leq S_{\mathbf{k}}} \frac{w_{\mathbf{y}}}{w_{y_1}}]$, where the order relation $S_{\mathbf{y}} < S_{\mathbf{k}}$ is according to the following definition.

\begin{definition}[Snake ordering]
\label{snake_ordering}
Let $\mathbf{k},\mathbf{k}'\in [[0,n-1]]^2$, with $\mathbf{k}=(x_1, x_2),\mathbf{k}'=(x'_1, x'_2)$ be distinct.
The snake ordering is defined as follows
\begin{itemize}
    \item if $x_2<x'_2$, then $\mathbf{k}<\mathbf{k}'$;
    \item if $x_2=x'_2$ and $x_2\in 2\mathbb{N}_0$, then $\mathbf{k}<\mathbf{k}'$ if $x_1<x'_1$ according to the snake ordering;
    \item if $x_2=x'_2$ and $x_2\in (2\mathbb{N}_0 + 1)$, then $\mathbf{k}<\mathbf{k}'$ if $x_1>x'_1$ according to the snake ordering.
\end{itemize}
\end{definition}

Finally, the transport map is given by
\[x \rightarrow \bigg( f_{\text{marg}}(g_k(x)) , \sum_{j=1}^{n} \Big(h_s^{\Delta_{jn}}(x) + \sum_{i=1}^{n-1} g_s^{\Delta_{i+jn}}(x) \Big) \bigg). \]
For a corresponding illustration, see Figure \ref{fig:second_method}. 

\begin{figure}
  \includegraphics[width=\linewidth]{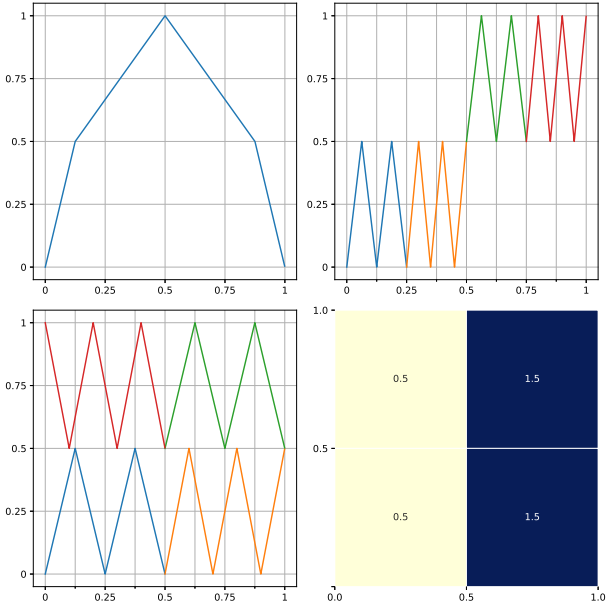}
  \caption{Generating the $2$-D histogram distribution using the alternative method. Top-left---the function $f_{\text{marg}}(g(x))$, top-right---the function $z(x)=g_2^{[0,1/4]}(x) + h_2^{[1/4,1/2]}(x) + g_2^{[1/2,3/4]}(x) + g_2^{[3/4,1]}(x)$, bottom-left---plot of the map $x \rightarrow (f_{\text{marg}}(g(x)), z(x))$, bottom-right---heatmap of the generated distribution.}
  \label{fig:second_method}
\end{figure}

\section{Conclusion}

The results in this paper show that every $d$-dimensional Lipschitz-continuous target distribution (under mild conditions on its pdf) can be generated through deep ReLU networks out of a one-dimensional uniform input distribution. What is more, this is possible without incurring a cost---in terms of approximation error measured in Wasserstein-distance---relative to generating the $d$-dimensional target distribution from $d$ independent random variables. This is accomplished through a two-stage approach, first generating a histogram distribution 
and then showing that increasing the histogram resolution drives the approximation error to zero while
the corresponding network connectivity scales no faster than the number of parameters in the class of histogram distributions considered. Concretely, this means that the generating network we devise has minimum possible connectivity scaling. We finally note that all the constructions in this paper employ 
histogram distributions of uniform tile size. As deep ReLU networks can generate histogram distributions of general tile sizes, it is likely that the constants in the bounds on the connectivity of the generating networks can be improved.

\section{Omitted proofs}

\subsection{Proof of Theorem \ref{line-wise_theorem}}
 \begin{proof}
  Let $p_X(x)$ be the marginal corresponding to $p_{X,Y}(x,y)$ and note that $p_{X}(x)$ is in $\mathcal{E}[0,1]^1_n$ and has weights $w_k$, $k \in [[0,n-1]]$. Define the map $M$ as follows $M: [0,1] \rightarrow [0,1]^2$,
  \[M:x \rightarrow (x, f(g_s(x))),\]
  where $g_s$ is an $s$-order sawtooth function according to Equation \ref{g_s_def} and $f(x)$ is defined according to Theorem \ref{1d_theorem} such that $f\#U = p_X(x)$. Fix $s\in \mathbb{N}$, take an arbitrary $r \in [[0,2^{s-1}-1]]$, and consider $f(g_s(x))$ on the interval $P_r = [\frac{r}{2^{s-1}}, \frac{r+1}{2^{s-1}}]$. By Lemma \ref{p_saw}, $f(g_s(x)) = f(g(2^{s-1}x-r)), \forall x \in P_r$. Now, let $c^{r}_{k,k_1} = [r 2^{-s+1}, (r+1)2^{-s+1}] \times [k/n + k_1 2^{-s+1}/n, k/n + (k_1+1)2^{-s+1}/n]$,
  $r, k_1 \in [[0,2^{s-1}-1]],k \in [[0,n-1]]$.
By Lemma \ref{p_saw0}, we have for all $k_1,k$,
\begin{equation}
\label{eq:line}
(M\#U(P_r)) (x\in c^{r}_{k,k_1}) = p_{X,Y}((x,y) \in c^{r}_{k,k_1}).
\end{equation}
Since $p_{X,Y}((x,y) \in c^{r}_{k,k_1}) = \frac{w_k}{n^2 2^{s-1}}$, for all $r \in [[0,2^{s-1}-1]]$, independently of $r$, by Lemma \ref{p_saw}, Equation \ref{eq:line} holds for all intervals $P_r$, $r \in [[0,2^{s-1}-1]]$. We have hence established that for all $r,k,k_1$, the map $M$ distributes probability mass to each of the rectangles $c^{r}_{k,k_1}$
according to $p_{X,Y}((x,y) \in c^{r}_{k,k_1})$. We refer to Figure \ref{fig:map} for a visualization of the transport map $M$.
Since $|x-y|\leq 2^{-s+1} \sqrt{1+\frac{1}{n}} \leq 2^{-s+3/2}$ for any two points in a rectangle of dimensions $(2^{-s+1} \times n^{-1}2^{-s+1})$, there exists a coupling $\pi$ that,
in each $c^{r}_{k,k_1}$,
associates points between $p_{X,Y}(x,y)$ and $M \#U$ owing to which we have
  \[W(M \#U, p_{X,Y}(x,y)) \leq \int_{[0,1]^2} 2^{-s+3/2} d(x,y) = \frac{2\sqrt{2}}{2^s}.\]
  It remains to show how the transport map
\[x \rightarrow (x, f(g_s(x))) \]
can be realized through a ReLU network.

We start by noting that the function $f(x) = \sum_{i=1}^n a_i \max\big(0, x - b_i\big)$ can be realized through the network 
$\Phi_1 \in \NN_{1,1}$ with $\Phi_1(x) = \sum_{i=1}^n a_i \rho(x - b_i)$, $\mathcal{M}(\Phi_1)\leq 3n$, and $\mathcal{L}(\Phi_1)=2$. The network $\Psi^s_g(x)$ realizing $g_s(x)$ is in $\NN_{1,1}$ with $\mathcal{M}(\Psi^s_g) = 11s-3$ and $\mathcal{L}(\Psi^s_g)=s+1$. It follows by Lemma II.3 in \cite{deep-it-2019} that $\Psi^f_s = \Phi_1(\Psi^s_g)$ is in $\NN_{1,1}$, with $\mathcal{M}(\Psi^f_s) \leq 22s+6n-6$ and $\mathcal{L}(\Psi^f_s) = s+3$.  The network $\Phi_2(x) = \rho(x) - \rho(-x) = x$ is in $\NN_{1,1}$ with $\mathcal{M}(\Phi_2) = 4$ and $\mathcal{L}(\Phi_2) = 2$. By Lemma II.4 in \cite{deep-it-2019}, there exists a network $\tilde{\Phi}_2(x)=\Phi_2(x)$ with $\mathcal{M}(\tilde{\Phi}_2) \leq 2s+8$ and $\mathcal{L}(\tilde{\Phi}_2) = s+3$. Finally, parallelizing $\tilde{\Phi}_2$ and $\Psi^f_s$ using Lemma A.7 in \cite{deep-it-2019}, we obtain the network $\Psi = (\tilde{\Phi}_2, \Psi^f_s)$, $\Psi \in \NN_{1,2}$, with $\mathcal{M}(\Psi)\leq 6n+24s+2$ and $\mathcal{L}(\Psi) = s+3$, implementing the desired transport map $x \rightarrow (x, f(g_s(x)))$.
\end{proof}

\subsection{Proof of Theorem \ref{2d_theorem}}
\begin{proof}
Let $I_i=[i/n,(i+1)/n]$ for $i \in [[0,n-1]]$ and let the weights of $p_{X,Y}(x,y)$ be given by $w_{k_1,k_2}$. Then, for every $i \in [[0,n-1]]$, consider the distribution $p_{Y}^i(y)= p_{Y}(y| x \in [i/n,(i+1)/n]) \in \mathcal{E}[0,1]^1_n$ with weights $w^i_{k} = \frac{nw_{i,k}}{\sum_{j=0}^{n-1}w_{j,k}}$, for $k \in [[0,n-1]]$, and let $f_i(x)$ be the corresponding piecewise linear function according to Theorem \ref{1d_theorem} such that $f_i\#U= p_{Y}^i$.
It follows from Definition \ref{d_dim_dist}, by integrating over $y$, that the marginal $p_X(x) \in \mathcal{E}[0,1]^1_n$ has weights $w_i = \sum_{j=0}^{n-1} w_{i,j}/n$, and we denote the piecewise linear function generating it according to Theorem \ref{1d_theorem} as $f_{\text{marg}}(x)$, i.e.,
$f_{\text{marg}}\#U=p_X$.
Take an arbitrary $r \in [[0,n-1]]$, fix $s \in \N$,  and consider the following transport map
\begin{equation}
    \label{2d_map}
    M: x \rightarrow \Bigg(f_{\text{marg}}(x), \sum_{i=0}^{n-1} f_i(g_s(nf_{\text{marg}}(x)-i)) \Bigg)
\end{equation}
on the interval $P_r:=[ \frac{1}{n} \sum_{j=0}^{r-1} w_j,  \frac{1}{n} \sum_{j=0}^{r} w_j]$. For $x\in P_r$, $f_{\text{marg}}(x) \in [r/n, (r+1)/n]$ and by Theorem \ref{1d_theorem} its explicit form is given by $f_{\text{marg}}(x) = \frac{x}{w_r} - \frac{\sum_{j=0}^{r-1}w_j}{nw_r} + \frac{r}{n}$. 
Therefore, $(nf_{\text{marg}}(x)-i) \in [r-i,r-i+1]$ and 
$f_i(g_s(nf_{\text{marg}}(x)-i)) = 0$, when $i\neq r$, as $g_s(x)=0, \forall x \notin [0,1]$. 
For $x\in P_r$, the transport map in Equation \ref{2d_map} hence becomes \[x \rightarrow \bigg(\frac{x}{w_r} - \frac{\sum_{j=0}^{r-1}w_j}{nw_r} + \frac{r}{n},p_r\Big(g_s\Big(\frac{nx- \sum_{j=0}^{r-1}w_j}{w_r}\Big)\Big)\bigg).\]
Now, let $c^{r,r_1}_{k,k_1} = [r/n + r_1 2^{-s+1}/n, r/n + (r_1+1)2^{-s+1}/n] \times [k/n + k_1 2^{-s+1}/n, k/n +
(k_1+1)2^{-s+1}/n]$, $r_1, k_1 \in [[0,2^{s-1}-1]],k \in [[0,n-1]]$. 
The square $c^{r,r_1}_{k,k_1}$ has area 
$\frac{2^{-2s+2}}{n^2}$
and $p_{X,Y}((x,y) \in c^{r,r_1}_{k,k_1}) = \frac{w_{r,k}}{2^{2s-2}n^2}$. Combining Lemmas \ref{p_saw0} and \ref{p_saw}, we obtain that for all $r_1,k_1,k$,
\begin{equation*}
\begin{aligned}
&(M\#U(P_r)) (x\in c^{r,r_1}_{k,k_1}) = \frac{w_r}{2^{s-1}n} \cdot \frac{w^r_k}{2^{s-1}n} \\
&= \frac{\sum_{j=0}^{n-1} w_{r,j}}{2^{2s-2}n^3} \cdot  \frac{nw_{r,k}}{\sum_{j=0}^{n-1}w_{r,j}} = \frac{w_{r,k}}{2^{2s-2}n^2}\\ &= p_{X,Y}((x,y) \in c^{r,r_1}_{k,k_1}).
\end{aligned}
\end{equation*} 
In summary, we found that $(M\#U(P_r)) (x \in c^{r,r_1}_{k,k_1}) = p_{X,Y}((x,y) \in c^{r,r_1}_{k,k_1})$, for arbitrary $r \in [[0,n-1]]$.
This establishes that for all $r,r_1,k,k_1$, the map $M$ distributes 
probability mass to each of the squares $c^{r,r_1}_{k,k_1}$ of area $\frac{2^{-2s+2}}{n^2}$ according to $p_{X,Y}((x,y) \in c^{r,r_1}_{k,k_1})$. We refer to Figure \ref{fig:map_comb} for a visualization of the corresponding transport map $M$.

Since $|x-y|\leq 2^{-s+3/2}/n$ for any two points in a box of
size $(n^{-1}2^{-s+1} \times n^{-1}2^{-s+1})$,
it follows that there exists a coupling $\pi$ between $p_{X,Y}(x,y)$ and $M\#U$ owing to which
  \[W(M \#U, p_{X,Y}(x,y)) \leq \frac{2\sqrt{2}}{n2^s}.\]

It remains to devise a ReLU network realizing the transport map in Equation \ref{2d_map}.

The functions $f_{i}(x)$ can be implemented
through networks $\Phi^i_1 \in \NN_{1,1}$ with
$\Phi^i_1(x) = \sum_{\ell=1}^n a_{\ell} \rho(x - b_{\ell})$, $\mathcal{M}(\Phi^i_1) \leq 3n$, and $\mathcal{L}(\Phi^i_1) =2$. 
We then note that $f_{\text{marg}}(x)$ can be realized through a network $\Phi_2 \in \NN_{1,1}$ with
$\mathcal{M}(\Phi_2) \leq 3n$ and $\mathcal{L}(\Phi_2) = 2$. The networks $\Phi^i_2(x)$ implementing $nf_{\text{marg}}(x)-i$ are in $\NN_{1,1}$ and have $\mathcal{M}(\Phi^i_2) \leq 4n$, $\mathcal{L}(\Phi^i_2) = 2$, and the network $\Psi^s_g(x) = g_s(x)$ is in $\NN_{1,1}$ with $\mathcal{M}(\Psi^s_g) = 11s-3$ and $\mathcal{L}(\Psi^s_g) =s+1$. By Lemma II.3 in \cite{deep-it-2019}, it follows that the networks $\Psi^i_s = \Phi^i_1(\Psi^s_g(\Phi^i_2))$ are in $\NN_{1,1}$ with $\mathcal{M}(\Psi^i_s) \leq 20n+44s-12$ and $\mathcal{L}(\Psi^i_s) = s+5$. By Lemma II.6 in \cite{deep-it-2019}, the network $\Psi^{\Sigma}=\sum_{i=0}^{n-1} \Psi^i_s$ realizing $\sum_{i=0}^{n-1} f_i\Big(g_s\Big(nf_{\text{marg}}(x)-i\Big)\Big)$ is in $\NN_{1,1}$ with $\mathcal{M}(\Psi^{\Sigma}) \leq 20n^2+44ns-12$ and $\mathcal{L}(\Psi^{\Sigma}) = s+5$. Thanks to Lemma II.4 in \cite{deep-it-2019}, there exists a network $\tilde{\Phi}_2(x)=\Phi_2(x)$ in $\NN_{1,1}$ with $\mathcal{M}(\tilde{\Phi}_2) \leq 4n+2s+6$ and $\mathcal{L}(\tilde{\Phi}_2) = s+5$. Parallelizing $\tilde{\Phi}_2$ and $\Psi^{\Sigma}$ using Lemma A.7 in \cite{deep-it-2019}, we obtain the network $\Psi = (\tilde{\Phi}_2, \Psi^{\Sigma})$, $\Psi \in \NN_{1,2}$, with $\mathcal{M}(\Psi) \leq 20n^2+44ns+4n+2s-6<88(n^2+ns)$ and $\mathcal{L}(\Psi) = s+5$, and realizing the transport map \[x \rightarrow \Bigg(f_{\text{marg}}(x), \sum_{i=0}^{n-1} f_i\Big(g_s\Big(nf_{\text{marg}}(x)-i)\Big)\Big) \Bigg).\qedhere \]
\end{proof}

\bibliography{references}

\begin{thebibliography}{19}
\providecommand{\natexlab}[1]{#1}
\providecommand{\url}[1]{\texttt{#1}}
\expandafter\ifx\csname urlstyle\endcsname\relax
  \providecommand{\doi}[1]{doi: #1}\else
  \providecommand{\doi}{doi: \begingroup \urlstyle{rm}\Url}\fi

\bibitem[Bailey \& Telgarsky(2018)Bailey and Telgarsky]{Bailey2019}
Bailey, B. and Telgarsky, M.~J.
\newblock Size-noise tradeoffs in generative networks.
\newblock In Bengio, S., Wallach, H., Larochelle, H., Grauman, K.,
  Cesa-Bianchi, N., and Garnett, R. (eds.), \emph{Advances in Neural
  Information Processing Systems 31}, pp.\  6489--6499. Curran Associates,
  Inc., 2018.
\newblock URL \url{https://arxiv.org/abs/1810.11158}.

\bibitem[{Barron}(1993)]{barron1993}
{Barron}, A.~R.
\newblock Universal approximation bounds for superpositions of a sigmoidal
  function.
\newblock \emph{IEEE Transactions on Information Theory}, 39\penalty0
  (3):\penalty0 930--945, May 1993.
\newblock ISSN 1557-9654.
\newblock \doi{10.1109/18.256500}.
\newblock URL \url{https://ieeexplore.ieee.org/document/256500}.

\bibitem[Bowman et~al.(2016)Bowman, Vilnis, Vinyals, Dai, Jozefowicz, and
  Bengio]{bowman2015generating}
Bowman, S.~R., Vilnis, L., Vinyals, O., Dai, A., Jozefowicz, R., and Bengio, S.
\newblock Generating sentences from a continuous space.
\newblock In \emph{Proceedings of The 20th {SIGNLL} Conference on Computational
  Natural Language Learning}, pp.\  10--21, Berlin, Germany, August 2016.
  Association for Computational Linguistics.
\newblock \doi{10.18653/v1/K16-1002}.
\newblock URL \url{https://www.aclweb.org/anthology/K16-1002}.

\bibitem[Box \& Muller(1958)Box and Muller]{box1958}
Box, G. E.~P. and Muller, M.~E.
\newblock A note on the generation of random normal deviates.
\newblock \emph{Ann. Math. Statist.}, 29\penalty0 (2):\penalty0 610--611, 06
  1958.
\newblock \doi{10.1214/aoms/1177706645}.
\newblock URL \url{https://doi.org/10.1214/aoms/1177706645}.

\bibitem[Clement \& Desch(2008)Clement and Desch]{triWass}
Clement, P. and Desch, W.
\newblock An elementary proof of the triangle inequality for the {W}asserstein
  metric.
\newblock \emph{Proceedings of the American Mathematical Society - PROC AMER
  MATH SOC}, 136:\penalty0 333--340, 01 2008.
\newblock \doi{10.1090/S0002-9939-07-09020-X}.
\newblock URL
  \url{https://www.ams.org/journals/proc/2008-136-01/S0002-9939-07-09020-X/}.

\bibitem[Devroye(1986)]{distgenBook}
Devroye, L.
\newblock Sample-based non-uniform random variate generation.
\newblock In \emph{Proceedings of the 18th Conference on Winter Simulation},
  WSC ’86, pp.\  260–265, New York, NY, USA, 1986. Association for
  Computing Machinery.
\newblock ISBN 0911801111.
\newblock \doi{10.1145/318242.318443}.
\newblock URL \url{https://doi.org/10.1145/318242.318443}.

\bibitem[Elbrächter et~al.(2019)Elbrächter, Perekrestenko, Grohs, and
  Bölcskei]{deep-it-2019}
Elbrächter, D., Perekrestenko, D., Grohs, P., and Bölcskei, H.
\newblock Deep neural network approximation theory.
\newblock \emph{IEEE Transactions on Information Theory}, 2019.
\newblock URL \url{http://www.mins.ee.ethz.ch/pubs/p/deep-it-2019}.
\newblock submitted.

\bibitem[Gibbs \& Su(2002)Gibbs and Su]{Gibbs2002}
Gibbs, A.~L. and Su, F.~E.
\newblock On choosing and bounding probability metrics.
\newblock \emph{International Statistical Review}, 70\penalty0 (3):\penalty0
  419--435, 2002.
\newblock \doi{10.1111/j.1751-5823.2002.tb00178.x}.
\newblock URL
  \url{https://onlinelibrary.wiley.com/doi/abs/10.1111/j.1751-5823.2002.tb00178.x}.

\bibitem[Goodfellow et~al.(2014)Goodfellow, Pouget-Abadie, Mirza, Xu,
  Warde-Farley, Ozair, Courville, and Bengio]{NIPS2014_5423}
Goodfellow, I., Pouget-Abadie, J., Mirza, M., Xu, B., Warde-Farley, D., Ozair,
  S., Courville, A., and Bengio, Y.
\newblock Generative adversarial nets.
\newblock In Ghahramani, Z., Welling, M., Cortes, C., Lawrence, N.~D., and
  Weinberger, K.~Q. (eds.), \emph{Advances in Neural Information Processing
  Systems 27}, pp.\  2672--2680. Curran Associates, Inc., 2014.
\newblock URL
  \url{http://papers.nips.cc/paper/5423-generative-adversarial-nets.pdf}.

\bibitem[Karras et~al.(2019)Karras, Laine, and Aila]{karras2018stylebased}
Karras, T., Laine, S., and Aila, T.
\newblock A style-based generator architecture for generative adversarial
  networks.
\newblock \emph{2019 IEEE/CVF Conference on Computer Vision and Pattern
  Recognition (CVPR)}, pp.\  4396--4405, 2019.
\newblock URL \url{https://arxiv.org/abs/1812.04948}.

\bibitem[Kingma \& Welling(2014)Kingma and Welling]{Welling2014:Gan}
Kingma, D. and Welling, M.
\newblock Auto-encoding variational bayes.
\newblock \emph{Proceedings of the 2nd International Conference on Learning
  Representations (ICLR)}, 2014.
\newblock URL \url{https://arxiv.org/abs/1312.6114}.

\bibitem[Lee et~al.(2017)Lee, Ge, Ma, Risteski, and Arora]{lee2017ability}
Lee, H., Ge, R., Ma, T., Risteski, A., and Arora, S.
\newblock On the ability of neural nets to express distributions.
\newblock In Kale, S. and Shamir, O. (eds.), \emph{Proceedings of the 30th
  Conference on Learning Theory, {COLT} 2017, Amsterdam, The Netherlands, 7-10
  July 2017}, volume~65 of \emph{Proceedings of Machine Learning Research},
  pp.\  1271--1296. {PMLR}, 2017.
\newblock URL \url{http://proceedings.mlr.press/v65/lee17a.html}.

\bibitem[Lu \& Lu(2020)Lu and Lu]{lu2020universal}
Lu, Y. and Lu, J.
\newblock A universal approximation theorem of deep neural networks for
  expressing distributions.
\newblock \emph{arXiv preprint arXiv:2004.08867}, 2020.
\newblock URL \url{https://arxiv.org/abs/2004.08867}.

\bibitem[Peyré \& Cuturi(2019)Peyré and Cuturi]{MAL-073}
Peyré, G. and Cuturi, M.
\newblock Computational optimal transport.
\newblock \emph{Foundations and Trends in Machine Learning}, 11\penalty0
  (5-6):\penalty0 355--607, 2019.
\newblock ISSN 1935-8237.
\newblock \doi{10.1561/2200000073}.
\newblock URL \url{http://dx.doi.org/10.1561/2200000073}.

\bibitem[Radford et~al.(2016)Radford, Metz, and
  Chintala]{radford2015unsupervised}
Radford, A., Metz, L., and Chintala, S.
\newblock Unsupervised representation learning with deep convolutional
  generative adversarial networks.
\newblock In Bengio, Y. and LeCun, Y. (eds.), \emph{4th International
  Conference on Learning Representations, {ICLR} 2016, San Juan, Puerto Rico,
  May 2-4, 2016, Conference Track Proceedings}, 2016.
\newblock URL \url{http://arxiv.org/abs/1511.06434}.

\bibitem[Telgarsky(2016)]{telgarsky2016benefits}
Telgarsky, M.
\newblock Benefits of depth in neural networks.
\newblock In Feldman, V., Rakhlin, A., and Shamir, O. (eds.), \emph{29th Annual
  Conference on Learning Theory}, volume~49 of \emph{Proceedings of Machine
  Learning Research}, pp.\  1517--1539, Columbia University, New York, New
  York, USA, 23--26 Jun 2016. PMLR.
\newblock URL \url{http://proceedings.mlr.press/v49/telgarsky16.html}.

\bibitem[Tolstikhin et~al.(2018)Tolstikhin, Bousquet, Gelly, and
  Sch{\"o}lkopf]{Tolstikhin2018Wassauto}
Tolstikhin, I., Bousquet, O., Gelly, S., and Sch{\"o}lkopf, B.
\newblock Wasserstein auto-encoders.
\newblock In \emph{6th International Conference on Learning Representations
  (ICLR)}, May 2018.
\newblock URL \url{https://openreview.net/forum?id=HkL7n1-0b}.

\bibitem[Villani(2008)]{villani2008optimal}
Villani, C.
\newblock \emph{Optimal transport: Old and new}, volume 338.
\newblock Springer Science \& Business Media, 2008.
\newblock URL \url{https://www.springer.com/de/book/9783540710493}.

\bibitem[Xu et~al.(2018)Xu, Ren, Lin, and Sun]{xu2018dpgan}
Xu, J., Ren, X., Lin, J., and Sun, X.
\newblock Diversity-promoting {GAN}: A cross-entropy based generative
  adversarial network for diversified text generation.
\newblock In \emph{Proceedings of the 2018 Conference on Empirical Methods in
  Natural Language Processing}, pp.\  3940--3949, Brussels, Belgium,
  October-November 2018. Association for Computational Linguistics.
\newblock \doi{10.18653/v1/D18-1428}.
\newblock URL \url{https://www.aclweb.org/anthology/D18-1428}.

\end{thebibliography}
\bibliographystyle{icml2020}

\newpage

\end{document}